\documentclass{article} % For LaTeX2e

\usepackage{iclr2026_conference}
\iclrfinalcopy  

\usepackage{times}

\usepackage{amsmath,amsfonts,bm}

\def\eqref#1{equation~\ref{#1}}
\def\1{\bm{1}}

\DeclareMathAlphabet{\mathsfit}{\encodingdefault}{\sfdefault}{m}{sl}
\SetMathAlphabet{\mathsfit}{bold}{\encodingdefault}{\sfdefault}{bx}{n}

\usepackage{hyperref}
\usepackage{url}

\usepackage[utf8]{inputenc} % allow utf-8 input
\usepackage[T1]{fontenc}    % use 8-bit T1 fonts
\usepackage{url}            % simple URL typesetting
\usepackage{booktabs}       % professional-quality tables
\usepackage{amsfonts}       % blackboard math symbols
\usepackage{nicefrac}       % compact symbols for 1/2, etc.
\usepackage{microtype}      % microtypography
\usepackage{xcolor}         % colors
\usepackage{amsmath}
\usepackage{algorithm}
\usepackage{algpseudocode}
\usepackage{tikz}
\usetikzlibrary{shapes, positioning,  arrows.meta,positioning,fit,calc,backgrounds,decorations.pathreplacing,patterns.meta}
\usepackage{subcaption}
\usepackage{listings}
\usepackage{hyperref}
\usepackage{dcolumn}
\usepackage{colortbl}
\usepackage{graphicx}
\usepackage{amssymb}
\usepackage{xcolor}
\usepackage{pifont}
\usepackage{wrapfig}
\usepackage{siunitx} % for number alignment (optional)
\usepackage{pgfplots}
\usepackage{pgfplotstable}
\pgfplotsset{compat=1.18}
\usepackage{tcolorbox}
\usepackage{amsthm}   
\usepackage{sansmath} % for uniform sans math in nodes

\newtheorem{theorem}{Theorem}[section]   % “Theorem” numbered within sections
\newtheorem{lemma}[theorem]{Lemma}       % Lemmas share the same counter
\newtheorem{corollary}[theorem]{Corollary}

\definecolor{keywordcolor}{rgb}{0.13, 0.29, 0.53}
\definecolor{stringcolor}{rgb}{0.63, 0.13, 0.13}
\definecolor{commentcolor}{rgb}{0.13, 0.54, 0.13}

\newcolumntype{d}[1]{D{.}{.}{#1}} % For decimal alignment on "-"

\definecolor{orangefield}{rgb}{1.0,0.5,0}
\definecolor{cyanfield}{rgb}{0,0.5,0.5}
\definecolor{yellowfield}{rgb}{1.0,1.0,0.0}
\definecolor{violetfield}{rgb}{1.0,0.0,1.0}

\definecolor{greenfield}{rgb}{0.0,0.39,0.15}

\definecolor{orangegoal}{rgb}{0.75,0.37,0}
\definecolor{cyangoal}{rgb}{0.0,0.65,0.65}
\definecolor{yellowgoal}{rgb}{0.71,0.71,0}
\definecolor{violetgoal}{rgb}{0.7,0.0,0.7}

\definecolor{darkgreen1}{rgb}{0.0, 0.7, 0.0} % Medium dark green
\definecolor{darkgreen2}{rgb}{0.0, 0.3, 0.0} % Darker green
\definecolor{darkgreen3}{rgb}{0.0, 0.2, 0.0} % Very dark green

\definecolor{graybg}{RGB}{230, 230, 230}

\definecolor{trajA}{RGB}{34,139,34}     % green
\definecolor{trajB}{RGB}{255,140,0}     % dark orange
\definecolor{trajC}{RGB}{200,30,45}     % red
\definecolor{softgray}{RGB}{120,120,120}
\definecolor{lightplus}{RGB}{210,240,210}
\definecolor{lightminus}{RGB}{255,220,220}

\newcommand{\PlotW}{6.6cm}
\newcommand{\PlotH}{3.2cm}
\newcommand{\ColSep}{1.2cm}
\newcommand{\RowGap}{0.9cm}

\newcommand{\offA}{0.15}
\newcommand{\offB}{0.0}
\newcommand{\offC}{0.50} % "worst" remains at 0

\definecolor{C1}{RGB}{0,114,178}    % blue
\definecolor{C2}{RGB}{213,94,0}     % vermillion
\definecolor{C3}{RGB}{0,158,115}    % bluish green
\definecolor{SoftGray}{RGB}{120,120,120}
\definecolor{PlusTint}{RGB}{225,245,255}
\definecolor{MinusTint}{RGB}{255,235,238}
\definecolor{MirrorGrey1}{RGB}{200,200,200}
\definecolor{MirrorGrey2}{RGB}{175,175,175}
\definecolor{MirrorGrey3}{RGB}{150,150,150}
\newcommand{\cell}{0.62cm}

\title{Relative Value Learning}

\author{Marc Höftmann, Jan Robine \& Stefan Harmeling \\
Department of Computer Science, Technical University of Dortmund, Germany \\
Lamarr Institute for Machine Learning and Artificial Intelligence\\
\texttt{\{marc.hoeftmann,jan.robine,stefan.harmeling\}@tu-dortmund.de} \\
}

\begin{document}

\maketitle

\begin{abstract}
	
	In reinforcement learning, critics typically estimate absolute state values $V(s)$, estimating how good a particular situation is in isolation. However, it turns out that only differences in value are relevant for control.
	Motivated by this, we propose \emph{Relative Value Learning} (RV), a framework that learns value differences directly via an antisymmetric function $\Delta(s_i, s_j) = V(s_i) - V(s_j)$. We introduce a pairwise Bellman operator and prove it is a $\gamma$-contraction with a unique fixed point equal to the true value differences, derive well-posed $1$-step, $n$-step and $\lambda$-return targets and reconstruct generalized advantage estimation from pairwise differences to obtain an unbiased policy-gradient estimator (R-GAE). Beyond theoretical results, we integrate RV with PPO and achieve competitive performance on the Atari benchmark (49 ALE games) compared to standard PPO, indicating that relative value estimation is an effective alternative to absolute critics. Our code is available at {\small \url{https://github.com/Hauf3n/relative-value-learning}}.

%	In reinforcement learning (RL), critics traditionally learn absolute state values, estimating how good a particular situation is in isolation. Adding any constant to 
%	$V(s)$ leaves action preferences unchanged. Thus only value differences are relevant for decision making.
%	Motivated by this fact, we ask the question whether these differences can be learned directly. For this, we propose \emph{Relative Value Learning} (RV), a framework that considers antisymmetric value differences $\Delta(s_i, s_j) = V(s_i) - V(s_j)$. We define a new pairwise Bellman operator and prove it is a $\gamma$-contraction with a unique fixed point equal to the true value differences, derive well-posed $1$-step/$n$-step/$\lambda$-return targets and reconstruct generalized advantage estimation from pairwise differences to obtain an unbiased policy-gradient estimator (R-GAE). Besides rigorous theoretical contributions, we integrate RV with PPO and achieve competitive performance on the Atari benchmark (49 games, ALE) compared to standard PPO, indicating that relative value estimation is an effective alternative to absolute critics. Source code will be made available.
	
\end{abstract}

%	\begin{center}
	%		\textbf{Keywords:} TD Learning, Relative Value Learning, Actor-Critic
	%	\end{center}

\vspace{-4mm}
\section{Motivation}

\begin{wrapfigure}{r}{0.49\textwidth}
	\centering
	\boxed{\includegraphics[width=0.48\textwidth]{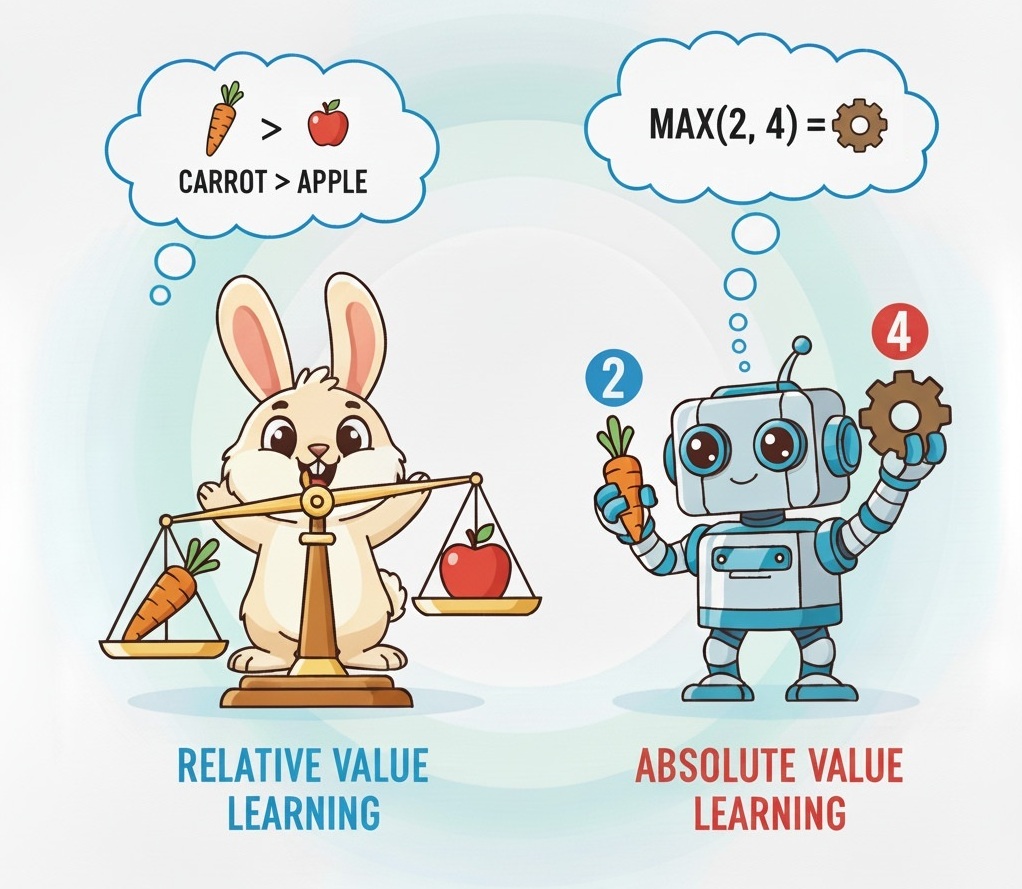}}
	\caption{\textbf{Relative Value vs. Absolute Value Learning.}. RV (left) learns value differences between states for decision making while AV (right) learns the value for each state in isolation and then decides for the best decision (e.g. by taking maximum in Q-learning). }
	\label{fig:sample}
\end{wrapfigure}

In control, actions are chosen by \emph{comparisons}, not by absolute magnitudes.
What matters is how good one state (or action) is \emph{relative} to another. Formally, for $\gamma<1$, $V^\pi$ is uniquely determined by the Bellman equation, but shifting it by any constant leaves advantages and greedy choices unchanged. This gauge freedom makes the absolute scale behaviorally meaningless: only differences matter. Advantages $A^\pi(s,a)$ or greedy action selection $\max_a Q^\pi(s,a)$ imply that absolute scales are not behaviorally meaningful. \emph{Optimal control depends not on absolute magnitudes but exclusively on relative differences.} Despite this, standard value-based RL trains a critic to approximate $V^\pi$ (or $Q^\pi$) and treats differences as a derived quantity. This introduces unnecessary degrees of freedom (the unpinned offset), invites drift under reward shaping or baseline changes, and can be ill-posed in settings where only comparisons or implicit feedback are given (e.g., preference-based or human-in-the-loop RL), precisely where the absolute scale is ambiguous while pairwise relations remain well-defined. A formulation that places \emph{relative} information at the center would match the invariances of the decision making problem itself.

%In value-based reinforcement learning (RL), the predominant paradigm is to either learn a value function 
%$V(s)$ or action-value function $Q(s,a)$ and then using it for control. Yet in decision making problems for biological and RL agents, the magnitude of a state’s value does not determine its behavior, but how that value compares to alternatives. Preferences, rankings, and choices are fundamentally relative. A decision is made because it is better than others, not because it crosses any particular absolute value threshold.
%This observation raises a conceptual question: \emph{if control depends only on differences of values, why learn absolute values at all?} Formally, under a fixed policy $\pi$ the value function is defined only up to an additive constant: if $V(s)$ satisfies the Bellman equation, then so does adding a constant $V(s) + c \:\: \forall c \in \mathbb{R}$. This \emph{gauge freedom} means that absolute values are not uniquely identifiable and that only \emph{value differences} carry invariant meaning. Consequently, policy gradient methods already exploit this property by subtracting baselines. \\

We therefore adopt the following viewpoint: make value differences the \emph{primary} learning objective.
Concretely, we learn an antisymmetric function $\Delta_\theta:S\times S\to\mathbb{R}$ with $\Delta_\theta(s_i,s_j)=-\Delta_\theta(s_j,s_i)$ that approximates $V^\pi(s_i)-V^\pi(s_j)$. Working directly with $\Delta_\theta$ eliminates the gauge degree of freedom by construction and aligns the critic with the invariants already exploited by policy-gradient methods through baselines. In addition, advantages can be reconstructed from pairwise differences without knowing the absolute value of any state, providing an unbiased policy-gradient estimator (R-GAE). This perspective is not just aesthetic. It enables a clean analytic foundation.

\textbf{Our Contributions:}
\begin{enumerate}
	\item \textbf{Pairwise Value Operator.} We formalize a Bellman operator on antisymmetric functions and prove $\gamma$-contraction with a fixed point equal to true value differences.
	\item \textbf{Value Targets.} We derive 1-step / $n$-step / $\lambda$-return targets using only observable rewards and non terminal pairwise terms, ensuring well-posed bootstrapping targets.
	\item \textbf{Relative GAE (R-GAE).} We show that GAE can be reconstructed from pairwise differences that also results in an \emph{unbiased} policy gradient estimator. In addition we derive the relationship between GAE and R-GAE and show that RV achieves competitive performance on Atari.
\end{enumerate}

%\paragraph{Paper structure.} Section 2 develops the RV framework (including operator, contraction, relative GAE, value targets). Section 3 covers relative value initialization, the training objective, and architecture. Section 4 presents experiments on the Atari benchmarks. Appendices provide the bootstrapping target derivations and empirical validation of the trajectory-constant baseline.

\section{Related Work}

Classical value-based RL trains absolute state or action values using Bellman operators, e.g. TD($\lambda$), DQN, Double DQN, Rainbow, and other actor critic variants \citep{sutton2018reinforcement,mnih2013playing,van2016deep, wang2016sample, hessel2018rainbow}
Distributional critics \citep{bellemare2017distributional, dabney2018implicit} restructure the target but continue to operate in an absolute space. Although adding a constant to $V^\pi$ or $Q^\pi$ leaves action preferences unchanged \citep{sutton2018reinforcement}, absolute critics still predict a scalar on an arbitrary scale. By contrast, RV removes the offset degree of freedom at the model level by learning antisymmetric value differences over state pairs, aligning the function class with the invariances of decision making. This invariance is a special case of policy-invariant reward transformations (potential-based shaping) \citep{ng1999policy}. Several methods estimate or emphasize advantages rather than values. Direct Advantage Estimation (DAE) \cite{pan2022direct} directly learns advantages $A^\pi(s,a)$ to bypass value learning. Dueling networks decompose learning $Q^\pi(s,a)=V^\pi(s)+A^\pi(s,a)$ to improve sample efficiency and robustness \cite{wang2016dueling}. Actor-critic variants such as A2C/A3C \cite{mnih2016asynchronous}, TRPO \cite{schulman2015trust} and clipped PPO \cite{schulman2017proximal} then use baselines and advantage estimates within trust-region-style updates. In contrast, RV directly learns value differences $\Delta(s_i,s_j)$ with an explicit pairwise Bellman operator to provide the actor \emph{relative advantages}. RV’s critic uses a siamese difference head that enforces antisymmetry and zero self-difference by design. Earlier work in stochastic control has also advocated learning value differences rather than absolute values. \citet{bertsekas1997differential} introduces differential training of rollout policies, approximating cost-to-go differences with TD-style methods.
In RL, pairwise objectives have appeared mainly in preference-based and human-in-the-loop RL \citep{christiano2017deep, leike1811scalable} or in inverse RL, but not as a Bellman-consistent value critic. Our work fills this gap.

\section{Relative Value Learning}
\label{sec:method}
We now present \emph{Relative Value Learning} (RV), which learns an \emph{antisymmetric} function over all state pairs with a neural network
\begin{equation}
	\Delta_\theta:\mathcal S\times\mathcal S\to\mathbb R,\qquad 
	\Delta_\theta(s_i,s_j)=-\Delta_\theta(s_j,s_i),
\end{equation}
that tries to approximate the \emph{value difference} $\Delta^\pi(s_i,s_j) \;:=\; V^\pi(s_i) - V^\pi(s_j)$, under a fixed policy $\pi$. As a result, RV avoids exact value estimates during training and integrates naturally with on-policy actor-critic methods (e.g., PPO) by supplying \emph{relative advantages} (R-GAE).

\subsection{Preliminaries}
\label{sec:prelims_objective}

We consider a discounted Markov decision process (MDP) $(S, A, P, r, \gamma)$ with state space $S$, action space $A$, transition function $P(s'\mid s,a)$, bounded reward function $r\!: S\times A \to [r_{\min}, r_{\max}]$, and discount factor $\gamma \in [0,1)$. Considering a stochastic policy $\pi(a\mid s)$, the value function under $\pi$ is
\begin{equation}
	V^{\pi}(s) \;:=\; \mathbb{E}_{\pi, P}\Bigl[\sum_{t=0}^{\infty} \gamma^t\, r(s_t,a_t)\,\Big|\, s_0=s\Bigr],
\end{equation}
which satisfies the Bellman equation
\begin{equation}
	V^{\pi}(s) \;=\; r_{\pi}(s) + \gamma\, \mathbb{E}_{s'\sim P^{\pi}(\cdot\mid s)}\bigl[V^{\pi}(s')\bigr], \quad r_{\pi}(s) := \mathbb{E}_{a\sim\pi(\cdot\mid s)}[r(s,a)].
	\label{eq:bellman_abs}
\end{equation}

\paragraph{Gauge Freedom and Relative Values.}  Shifted value functions do also satisfy a Bellman equation like Equation~\ref{eq:bellman_abs} (Lemma \ref{lem:constant-gauge} explains how to shape the reward). These shifts are irrelevant for control, since the advantages or greedy action selection are invariant.  Hence the absolute scale of $V^{\pi}$ is not meaningful, and only differences of values are invariant. We therefore define the \emph{pairwise value difference}
\begin{equation}
	\Delta^{\pi}(s_i,s_j) \;=\; V^{\pi}(s_i) - V^{\pi}(s_j),\quad (s_i,s_j)\in S\times S.
	\label{eq:def_rel_value}
\end{equation}
Note that it is antisymmetric $\Delta^{\pi}(s_i,s_j)=-\Delta^{\pi}(s_j,s_i)$, and satisfies $\Delta^{\pi}(s,s)=0$ for all $s$.

\paragraph{Pairwise Bellman Identity.} Fix $(s_i,s_j)\in S\times S$. Draw $s'_i\sim P^{\pi}(\cdot\mid s_i)$ and $s'_j\sim P^{\pi}(\cdot\mid s_j)$ \emph{independently} (only conditional on $s_i$ or $s_j$ respectively). Subtracting the Bellman equations of (\ref{eq:bellman_abs}) for $s_i$ and $s_j$ gives the recursive identity
\begin{equation}
	\Delta^{\pi}(s_i,s_j) \;=\; r_{\pi}(s_i) - r_{\pi}(s_j) 
	\; +\; \gamma\, \mathbb{E}_{\substack{s'_i\sim P^{\pi}(\cdot\mid s_i)\\ s'_j\sim P^{\pi}(\cdot\mid s_j)}}\bigl[\Delta^{\pi}(s'_i,s'_j)\bigr].
	\label{eq:pairwise_identity}
\end{equation}
Equation~\ref{eq:pairwise_identity} depends only on observable one-step rewards via $r_{\pi}$ and on pairwise differences at successors. It is invariant to any additive shift of $V^{\pi}$ as we have shown in Lemma \ref{lem:constant-gauge}.

\subsection{Pairwise Bellman Operator}
\label{subsec:pairwise-bellman-operator}

Fix a policy $\pi$. We work on the Banach space of bounded antisymmetric pairwise functions
\begin{align}
	\mathcal{F}
	\;:=\;
	\Bigl\{
	&\Delta : \mathcal{S}\times\mathcal{S}\to\mathbb{R}
	\;\big|\;
	\Delta(s_i,s_j)=-\Delta(s_j,s_i),\;
	\|\Delta\|_\infty<\infty
	\Bigr\},\\
	\:
	&\|\Delta\|_\infty \;=\; \sup_{(s_i,s_j)} |\Delta(s_i,s_j)|. \nonumber
\end{align}
In addition, let $\Delta r^\pi(s_i,s_j) := r^\pi(s_i)-r^\pi(s_j)$ be the immediate reward difference.

\paragraph{Operator form.}
Let $\widehat{\mathcal{P}}_\pi$ act on $\mathcal{F}$ by
\begin{equation}
	(\widehat{\mathcal{P}}_\pi \Delta)(s_i,s_j)
	\;=\;
	\mathbb{E}_{s'_i\sim P^\pi(\cdot|s_i),\, s'_j\sim P^\pi(\cdot|s_j)}
	\bigl[\Delta(s'_i,s'_j)\bigr].
	\label{eq:pairwise-bellman-identity}
\end{equation}

Define the \emph{pairwise Bellman operator} $T_\pi:\mathcal{F}\to\mathcal{F}$ by
\begin{equation}
	(T_\pi \Delta)(s_i,s_j)
	\;:=\;
	\Delta r^\pi(s_i,s_j)\;+\;\gamma\,(\widehat{\mathcal{P}}_\pi \Delta)(s_i,s_j).
	\label{eq:pairwise-bellman-operator}
\end{equation}
This coincides with the definition from Equation~(\ref{eq:pairwise_identity}) used in our method when $s'_i$ and $s'_j$ are drawn independently from $P^\pi(\cdot|s_i)$ and $P^\pi(\cdot|s_j)$, respectively.

\begin{theorem}[Contraction and uniqueness]
	\label{thm:contraction}
	For any $\Delta_1,\Delta_2\in\mathcal{F}$,
	\[
	\|T_\pi \Delta_1 - T_\pi \Delta_2\|_\infty
	\;\le\;
	\gamma\,\|\Delta_1-\Delta_2\|_\infty.
	\]
	Consequently, by the Banach fixed-point theorem, $T_\pi$ has a unique fixed point $\Delta^\pi\in\mathcal{F}$. Moreover, this fixed point equals the true value differences, i.e., $\Delta^\pi(s_i,s_j)=V^\pi(s_i)-V^\pi(s_j)$ for all $(s_i,s_j)\in\mathcal{S}\times\mathcal{S}$.
\end{theorem}

\begin{proof}
	The immediate difference $\Delta r^\pi$ cancels in $T_\pi \Delta_1 - T_\pi \Delta_2$, hence for each $(s_i,s_j)$,
	\begin{align}
		\bigl|(T_\pi \Delta_1 - T_\pi \Delta_2)(s_i,s_j)\bigr|
		\;=\;
		\gamma\,\Bigl|\mathbb{E}\bigl[(\Delta_1-\Delta_2)(s'_i,s'_j)\bigr]\Bigr|
		\;&\le\;
		\gamma\,\mathbb{E}\bigl[|(\Delta_1-\Delta_2)(s'_i,s'_j)|\bigr] \nonumber\\
		\;&\le\;
		\gamma\,\|\Delta_1-\Delta_2\|_\infty. \nonumber
	\end{align}
	Since this bound holds for all pairs $(s_i,s_j)$ thus taking the supremum over $(s_i,s_j)$ concludes with  $\|\mathcal{T}\Delta_1-\mathcal{T}\Delta_2\|_\infty\le\gamma\|\Delta_1-\Delta_2\|_\infty.$
\end{proof}
Related ‘relative’ value functions arise in average-reward MDPs, where values are defined only up to an additive constant and algorithms fix the gauge via relative value iteration \citep{abounadi2001learning,puterman2014markov,bertsekas2025neuro}.

\subsection{Relative Generalized Advantage Estimation (R-GAE)}
\label{sec:R-GAE}
In this section, we define the R-GAE estimator which is in essence analogous to GAE \citep{schulman2015high}. Furthermore, we derive the relationship between GAE and R-GAE and show that both estimators learn the same (optimal) policy. Then, R-GAE is combined with PPO \citep{schulman2017proximal} to provide \emph{relative advantages}.
Hereby, one of our key theoretical contributions is the fact that GAE can be reconstructed without knowing the exact value of a state.
\vspace{-3mm}
\paragraph{Relative GAE.} First, we construct \emph{relative values} $\tilde V_\theta$ for each state of an arbitrary environment rollout $(s_0,s_1,\dots,s_T)$, where $T$ denotes the rollout length. Note in our notation that $s_0$ is \emph{not} the environment's start state, but it can be any state. This notation simplifies the following descriptions. We define the relative value sequence by telescoping the differences:
\vspace{-1mm}
\begin{equation}
	\label{eq:val_init}
	\tilde V_\theta(s_0) := 0, \quad \tilde V_\theta(s_t) := \sum_{k=0}^{t-1} \Delta_\theta(s_{k+1},s_k) \quad (t\ge1).
\end{equation}
Note, that it is also possible to use a larger step size for calculating relative values, e.g. one can directly compute $ \tilde V_\theta(s_t) = \Delta_\theta(s_{t},s_0)$. If $\Delta_\theta=\Delta^\pi$, then the relationship becomes $\tilde V_\theta(s_t) = V^\pi(s_t)-V^\pi(s_0)$. Analogous to GAE, define relative TD residuals
\begin{equation}
	\label{eq:relative_td_residuals}
	\tilde \delta_t \;:=\; r_t + \gamma \tilde V_\theta(s_{t+1}) - \tilde V_\theta(s_t),
\end{equation}
\vspace{-1mm}
and finally construct the relative GAE similarly
\begin{equation}
	\tilde A_t \;:=\; \sum_{l=0}^{T-t}(\gamma\lambda)^l \tilde \delta_{t+l}.
\end{equation}
Essentially, RV acts as the critic and replaces GAE with R-GAE in the PPO clipping objective. Furthermore, we give additional theoretical insight about R-GAE in the following. In short, these insights conclude that:
\begin{enumerate}
	\item There exists a relationship between GAE and R-GAE that is given as $\tilde A_t \;=\; A_t + B_t$. For more details, see Lemma \ref{lem:gae-relation}. \\
	\item  The optimal policy for both estimators is identical. See Corollary \ref{cor:unbiased-relative}.
\end{enumerate} 
\vspace{2mm}
\begin{lemma}[Relationship between GAE and R-GAE]\label{lem:gae-relation}
	We call $C:= V^\pi(s_0)$ the trajectory constant. If $\Delta_\theta=\Delta^\pi$, then the following equality holds
	\begin{equation}
		\tilde A_t \;=\; A_t + B_t,
	\end{equation}
	where $A_t$ is the standard GAE computed from $V^\pi$, and
	$
	B_t \;=\; (1-\gamma) C \sum_{l=0}^{T-t} (\gamma\lambda)^l.
	$
\end{lemma}
\begin{proof}
	By construction of the relative values, we have that
	$\tilde V(s_t)=V^\pi(s_t)-C$ for all $t$. Hence inserting this form in Equation \ref{eq:relative_td_residuals} gives
	\begin{equation}
		\tilde\delta_t
		= r_t + \gamma\bigl(V^\pi(s_{t+1})-C\bigr) - \bigl(V^\pi(s_t)-C\bigr)
		= \underbrace{r_t + \gamma V^\pi(s_{t+1}) - V^\pi(s_t)}_{\delta_t}
		\;+\; (1-\gamma)\,C,
	\end{equation}
	where $\delta_t$ is the TD residual. Therefore, for each timestep the relative advantage is
	\[
	\tilde A_t
	= \sum_{l=0}^{T-t} (\gamma\lambda)^l \,\tilde\delta_{t+l}
	= \sum_{l=0}^{T-t} (\gamma\lambda)^l \,\delta_{t+l} 
	\;+\; (1-\gamma)\,C \sum_{l=0}^{T-t} (\gamma\lambda)^l
	= A_t + B_t,
	\]
\end{proof}
\paragraph{Gauge Invariance vs. Trajectory-constant Baseline.}  To get a better understanding of the trajectory constant $B_t$, we discuss its role and properties in Appendix \ref{sec:emperical_GAE_baseline}. Lemma \ref{lem:constant-gauge} states that adding any scalar $c \in R $ to $V^\pi$ (together with the corresponding reward shaping) leaves action preferences, advantages and pairwise value differences invariant; hence the absolute gauge of $V^\pi$ is behaviorally irrelevant.
Lemma \ref{lem:gae-relation}, however, shows when we reconstruct advantages from pairwise differences by anchoring to zero on each rollout, there exists indeed a difference quantified by $B_t$.
These statements are compatible and are not contradicting each other. Anchoring $\tilde{V}(s_0) = 0$ is a gauge fixing analogous to the differential (relative) value normalization used in average-reward MDPs \citep{abounadi2001learning,puterman2014markov,bertsekas2025neuro}. The only difference is scope: Ours is per-trajectory and theirs is global, so unbiasedness (see Corollary \ref{cor:unbiased-relative}) is preserved while a trajectory-constant $B_t$ may appear. 

As a remark, we can achieve pointwise equality with GAE, if we learned the \emph{non-antisymmetric} two-argument function
\begin{equation}
	\Delta_\gamma(s',s) := \gamma V(s') - V(s),
\end{equation}
so that the temporal-difference residual $\delta_t = r_t + \Delta_\gamma(s_{t+1},s_t)$ is modeled precisely. The corresponding Bellman operator is indeed a contractive self-map on the pairwise space, which can be derived analogous to Theorem \ref{thm:contraction}. However, we tried to apply $\Delta_\gamma(s',s)$ in practice, but the results are worse compared to R-GAE. Therefore, we continue to work with R-GAE and show in Section \ref{sec:mean-min-baseline} how to reduce the variance of $B_t$.

\begin{corollary}[Unbiasedness for policy gradient]\label{cor:unbiased-relative}
	Let $\nabla_\phi J(\phi):=\mathbb E_t\!\left[ \nabla_\phi\log \pi_\phi(a_t\mid s_t)\,A_t\right]$ denote the standard policy gradient (PG) under $\pi_\phi$. If the advantages in the score-function estimator are replaced by the relative advantages from Lemma~\ref{lem:gae-relation}, i.e.,
	\[
	\nabla_\phi \tilde J_{}(\phi)\;\ := \; \mathbb E_t\!\left[ \nabla_\phi\log \pi_\phi(a_t\mid s_t)\,\tilde A_t\right],
	\]
	then
	\begin{equation}\label{eq:unbiased}
		\nabla_\phi \tilde J_{}(\phi)
		\;=\;\nabla_\phi J(\phi).
	\end{equation}
\end{corollary}
\begin{proof}
	By Lemma~\ref{lem:gae-relation}, $\tilde A_t = A_t + B_t$ with $B_t = \frac{C \: \!\left(1-(\gamma\lambda)^{T-t+1}\right)}{1-\gamma\lambda}$ as trajectory constant, hence
	\begin{equation}
		\mathbb{E}_t\!\left[\nabla_\phi \log \pi_\phi(a_t\mid s_t)\,\tilde A_t\right]
		=\mathbb{E}_t\!\left[\nabla_\phi \log \pi_\phi(a_t\mid s_t)\,A_t\right]
		+\mathbb{E}_t\!\left[\nabla_\phi \log \pi_\phi(a_t\mid s_t)\,B_t\right].
	\end{equation}
	Condition on $s_t$ by the trajectory prefix up to time $t$ (so that $B_t$ does not depend on $a_t$) and then using the score-function identity gives
	\begin{equation}
		\mathbb{E}\!\left[\nabla_\phi \log \pi_\phi(a_t\mid s_t)\,B_t \;\middle|\; s_t\right]
		= B_t\,\mathbb{E}_{a_t\sim\pi_\phi(\cdot\mid s_t)}\!\left[\nabla_\phi \log \pi_\phi(a_t\mid s_t)\right]
		= 0.
	\end{equation}
	Taking expectations over $t$ yields
	\begin{equation}
		\mathbb{E}_t\!\left[\nabla_\phi \log \pi_\phi(a_t\mid s_t)\,\tilde A_t\right]
		= \mathbb{E}_t\!\left[\nabla_\phi \log \pi_\phi(a_t\mid s_t)\,A_t\right],
	\end{equation}
	which is precisely Equation~\ref{eq:unbiased}. \qedhere
\end{proof}

\subsection{Relative Value Targets}
\label{subsec:relative-value-targets}
When approximating the pairwise Bellman operator (Eq. \ref{eq:pairwise-bellman-operator}) from two sampled 1-step transitions
$
\tau_i = (s_i,a_i, r_i, d_i, s_{i+1}), \: \tau_j = (s_j, a_j, r_j, d_j, s_{j+1}),
$
terminal successor states make the naive bootstrap $\Delta(s_{i+1}, s_{j+1})$ ill-posed because absolute values are not available in our formulation. For example, if $s_{i+1}$ is a terminal state ($d_i=1$), then we need to calculate
\[
\Delta(s_{i+1}, s_{j+1}) = 0 - V(s_{j+1}) = -V(s_{j+1}),
\]

which is simply not accessible in that way. Therefore, we need to rearrange all bootstrapping targets in terms of observable rewards and non terminal pairwise differences $\Delta_\theta(\cdot,\cdot)$. In this way the targets get well-posed and are compatible with the operator $T_\pi$ on antisymmetric functions. Due to limited space, the complete formal derivation is given in Appendix~\ref{sec:rv_target_derivation}. 
\paragraph{1-step Target.}
Given the prediction $\Delta_\theta(s_i,s_j)$, the corresponding 1-step target takes the form
\begin{equation}
	\label{eq:rv-1step}
	y^{(1)}_{ij} \; :=\; (r_i - r_j) \;+\; \gamma \,\delta_{ij},
\end{equation}
where the bootstrap term $\delta_{ij}$ depends on the successor terminal flags $d_i,d_j \in \{0,1\}$:
\begin{equation}
	\label{eq:rv-1step-cases}
	\delta_{ij} \;=\;
	\begin{cases}
		\Delta_\theta(s_{i+1}, s_{j+1}), & \text{if } d_i = 0,\; d_j = 0,\\[2pt]
		\Delta_\theta(s_{i+1}, s_{j}) \;+\; r_j, & \text{if } d_i = 0,\; d_j = 1,\\[2pt]
		\Delta_\theta(s_{i}, s_{j+1}) \;-\; r_i, & \text{if } d_i = 1,\; d_j = 0,\\[2pt]
		\Delta_\theta(s_{i}, s_{j}) \;+\; r_j \;-\; r_i, & \text{if } d_i = 1,\; d_j = 1.
	\end{cases}
\end{equation}
When $d_i{=}d_j{=}1$, one may replace the last line of Equation~\ref{eq:rv-1step-cases} by $\delta_{ij}{=}0$ (both successors are absorbing with zero value) to reduce variance at ends of the two episodes. We use   $\delta_{ij}{=}0$ by default, so the derived fourth case in Appendix \ref{sec:rv_target_derivation} is optional.

\paragraph{N-step Target.}
To extend beyond 1-step temporal difference targets, we define the pairwise $n$-step case 
by considering two trajectories 
\[
\tau_i=\{(s_{i+k},r_{i+k},d_{i+k}, s_{i+k+1})\}_{k\ge 0},
\qquad
\tau_j=\{(s_{j+k},r_{j+k},d_{j+k}, s_{j+k+1})\}_{k\ge 0}.
\]
Then, the target becomes
\begin{equation}
	\label{eq:n-step-target}
	y^{(n)}_{ij} := \sum_{k=0}^{n-1} \gamma^k (r_{i+k} - r_{j+k}) + \gamma^n \Delta_\theta(s_{i+n}, s_{j+n}),
\end{equation}
with the assumption that neither trajectory terminates within the $n$-step window (i.e., $d_{i+k} = d_{j+k} = 0$ for $k < n$). If the trajectories have different length, then we take the minimum length. Note, that the final estimated difference $\Delta_\theta(s_{i+n}, s_{j+n})$ needs to consider the case distinction from Eq. \ref{eq:rv-1step-cases}.

\paragraph{$\lambda$-Return.}
Interpolating between high-bias/low-variance ($n{=}1$) and Monte-Carlo limits gives the pairwise $\lambda$-return
\begin{equation}
	\label{eq:rv-lambda}
	y^{(\lambda)}_{ij} \;:=\; (1-\lambda)\sum_{n=1}^{\infty} \lambda^{\,n-1}\, y^{(n)}_{ij}, \qquad \lambda \in [0,1],
\end{equation}
with the convention that $y^{(n)}_{ij}$ is truncated at the first encountered terminal using the case distinction from Equation \ref{eq:rv-1step-cases}.

\section{Relative Value Initialization}\label{sec:mean-min-baseline}

As demonstrated in Lemma \ref{lem:gae-relation}, the zero-anchor $\tilde V(s_0)=0$ (see Equation \ref{eq:val_init}) induces a trajectory-constant offset $B_t$ for R-GAE.
If the unknown $|C|$ is large (recall $C=V^\pi(s_0)$), then $B_t$ inflates the magnitude of $\tilde A_t$. It thereby can increase the variance of our policy-gradient estimate (see Appendix \ref{sec:emperical_GAE_baseline}). The increased variance makes the credit assignment problem harder and we therefore seek a data-dependent initialization that drives $\mathbb E_t[B_t]\approx 0$ over a collected batch. Intuitively speaking, our solution is to rank trajectories relative to each other which can be understood as anchoring. For a visual example, see Figure \ref{fig:relative_offsets}. 

\begin{figure}
	\begin{tikzpicture}[>=Latex, font=\footnotesize]
		
		% ===================== LEFT COLUMN: two plots =====================
		% BEFORE
		\begin{axis}[
			name=before,
			at={(0,0)}, anchor=south west,
			width=\PlotW, height=\PlotH,
			axis lines=left, xmin=0.0, xmax=1, ymin=-0.5, ymax=1.20,
			xtick={-1,10}, ytick={-10,0,10.0},
			tick style={black!55}, ticklabel style={black!70},
			xlabel={$t$}, ylabel={},
			title={\scshape Relative Values},
			title style={font=\scriptsize, yshift=1pt},
			clip=false
			]
			% zero line
			\addplot[SoftGray, densely dotted, domain=0:1] {0};
			% curves
			\addplot[C1, semithick, domain=0:1, samples=140] {1.30*(1 - exp(-3*x) + 0.2*sin(deg(6*x))} node[pos=1.01,anchor=south,inner sep=1pt, text=C1] {$\tau_1$};
			\addplot[C2, semithick, domain=0:1, samples=140] {0.60*(1 - exp(-4.0*x)) + 0.06*sin(deg(18*x))}
			node[pos=1.01,anchor=south,inner sep=1pt, text=C2] {$\tau_2$};
			\addplot[C3, semithick, domain=0:1, samples=140] {0.28*(1 - exp(0.4*x)) - 0.03*sin(deg(40*x))}
			node[pos=1.01,anchor=south,inner sep=1pt, text=C3] {$\tau_3$};
		\end{axis}
		
		% AFTER
		\begin{axis}[
			name=after,
			at={($(before.south west)+(0,-3.5cm)$)}, anchor=south west,
			width=\PlotW, height=\PlotH,
			axis lines=left, xmin=0.0, xmax=1, ymin=-0.5, ymax=1.40,
			xtick={-1,10}, ytick={-10,0,10},
			tick style={black!55}, ticklabel style={black!70},
			xlabel={$t$}, ylabel={},
			title={\scshape Relative Values with Offset},
			title style={font=\scriptsize, yshift=1pt},
			clip=false
			]
			\addplot[SoftGray, densely dotted, domain=0:1] {0};
			
			% shifted curves
			\addplot[C1, semithick, domain=0:1, samples=140] {(\offA) + 1.30*(1 - exp(-3*x) + 0.2*sin(deg(6*x))}
			node[pos=0.98,anchor=south,inner sep=1pt, text=C1] {$\tau_1$};
			\addplot[C2, semithick, domain=0:1, samples=140] {(\offB) + 0.60*(1 - exp(-4.0*x)) + 0.06*sin(deg(18*x))}
			node[pos=0.98,anchor=south,inner sep=1pt, text=C2] {$\tau_2$};
			\addplot[C3, semithick, domain=0:1, samples=140] {(\offC) + 0.28*(1 - exp(0.4*x)) - 0.03*sin(deg(40*x))}
			node[pos=0.98,anchor=north,inner sep=1pt, text=C3] {$\tau_3$};
			
			% save offsets (axis-local coords) for external connections
			\path (axis cs:0,\offA) coordinate (Aoff);
			\path (axis cs:0,\offB) coordinate (Boff);
			\path (axis cs:0,\offC) coordinate (Coff);
			
			% vertical offset markers at x=0
			\draw[C1, line width=0.9pt] (axis cs:0,0) -- (Aoff);
			\draw[C2, line width=0.9pt] (axis cs:0,0) -- (Boff);
			\draw[C3, line width=0.9pt] (axis cs:0,0) -- (Coff);
			\fill[C1] (Aoff) circle (1.2pt); \fill[C2] (Boff) circle (1.2pt); \fill[C3] (Coff) circle (1.2pt);
		\end{axis}
		
		\node[font=\scriptsize, black!100, anchor=west] at ($(before.north west)+(-0.9cm,0.2cm)$) {$\tilde V_\theta(s_t)$};
		
		\node[font=\scriptsize, black!100, anchor=west] at ($(before.north west)+(-0.9cm,-3.4cm)$) {$\bar V_\theta(s_t)$};
		
		%\node[font=\scriptsize, black!60, anchor=west] at ($(after.north west)+(-0.1cm,0.3cm)$) %{\scshape Relative Values with Offset};
		
		% ===================== RIGHT COLUMN: two-row pipeline =====================
		\coordinate (rightcol) at ($(before.north east)+(\ColSep,-0.2cm)$);
		
		% small panel heading
		%\node[font=\scriptsize, black!60, anchor=west] at ($(rightcol)+(0,0.15cm)$) {\scshape Baseline Derivation (3 trajectories)};
		
		% ---------- ROW 1: Delta -> m ----------
		% Matrix \Delta
		\begin{scope}[shift={($(rightcol)+(0.5cm,0.2)$)}, name=relativem]
			
			\draw[rounded corners=1.4pt, line width=0.6pt, black!70] (0,0) rectangle (3*\cell, -3*\cell);
			\foreach \i in {1,2} {
				\draw (0,-\i*\cell) -- (3*\cell,-\i*\cell);
				\draw (\i*\cell,0) -- (\i*\cell,-3*\cell);
			}
			
			\node at (0.5*\cell,-0.5*\cell) {$0$};
			\node[C1] at (1.5*\cell,-0.5*\cell) {$\Delta_{\textcolor{C1}{0}\textcolor{C2}{0}}$};
			\node[C1] at (2.5*\cell,-0.5*\cell) {$\Delta_{\textcolor{C1}{0}\textcolor{C3}{0}}$};
			\node[C2] at (0.5*\cell,-1.5*\cell) {$\Delta_{0\textcolor{C1}{0}}$};
			\node at (1.5*\cell,-1.5*\cell) {$0$};
			\node[C2] at (2.5*\cell,-1.5*\cell) {$\Delta_{0\textcolor{C3}{0}}$};
			\node[C3] at (0.5*\cell,-2.5*\cell) {$\Delta_{0\textcolor{C1}{0}}$};
			\node[C3] at (1.5*\cell,-2.5*\cell) {$\Delta_{0\textcolor{C2}{0}}$};
			\node at (2.5*\cell,-2.5*\cell) {$0$};
			
			% labels
			\node[C1] at (-0.26cm, -0.5*\cell) {$s_0$};
			\node[C2] at (-0.26cm, -1.5*\cell) {$s_0$};
			\node[C3] at (-0.26cm, -2.5*\cell) {$s_0$};
			\node[C1] at (0.5*\cell, 0.25cm) {$s_0$};
			\node[C2] at (1.5*\cell, 0.25cm) {$s_0$};
			\node[C3] at (2.5*\cell, 0.25cm) {$s_0$};
			
			\node[midway, yshift=7pt, font=\scriptsize] {$$};
			\node[midway, xshift=-10pt, font=\scriptsize, rotate=90] {$$};

			\node[font=\scriptsize, black!60, anchor=west] at ($(rightcol)+(0.3cm,0.9cm)$) {\scshape RV Differences};
			
		\end{scope}
		
		% Arrow to m (row-wise maxima) + step marker (1)
		\filldraw[white] ($(rightcol)+(0.1cm, 0.9cm)$) circle (6pt);
		\node[draw, circle, inner sep=1.6pt, black!70] at ($(rightcol)+(0.1cm, 0.9cm)$) {\scriptsize 1};
		
		\draw[-{Stealth[length=2.2mm]}, line width=0.7pt, SoftGray]
		($(rightcol)+(3*\cell+0.8cm,-0.8cm)$) -- ++(1.8cm, 0.0cm)
		node[midway, above, font=\tiny, black!70] {row mean};
		
		% Vector m
		\begin{scope}[shift={($(rightcol)+(3*\cell+2.8cm,0.2cm)$)}]
			\draw[rounded corners=1.2pt, line width=0.6pt, black!70] (0,0) rectangle (\cell+0.25cm,-3*\cell);
			\node[] at (0.7*\cell,-0.5*\cell) {$O(\textcolor{C1}{s_0})$};
			\node[] at (0.7*\cell,-1.5*\cell) {$O(\textcolor{C2}{s_0})$};
			\node[] at (0.7*\cell,-2.5*\cell) {$O(\textcolor{C3}{s_0})$};
			\path (0.5*\cell,-1.5*\cell) coordinate (mcenter);
		\end{scope}
		
		% Vertical arrow: mean--min centering (step 2)
		\filldraw[white] ($(rightcol)+(3*\cell+2.15cm-0.32cm,0.9cm)$) circle (6pt);
		\node[draw, circle, inner sep=1.6pt, black!70] at ($(rightcol)+(3*\cell+2.15cm-0.32cm,0.9cm)$) {\scriptsize 2};
		
		\node[font=\scriptsize, black!60, anchor=west] at ($(rightcol)+(3*\cell+2.15cm-0.1cm,0.9cm)$) {\scshape Estimate Offset};
		
		\coordinate (rowtwo) at ($(rightcol)+(0,-3*\cell-\RowGap)$);
		%\draw[-{Stealth[length=2.2mm]}, line width=0.7pt, SoftGray](mcenter) -- ($(rowtwo)+(1.5*\cell,0)$);
		%node[midway, right, xshift=4pt, font=\scriptsize, black!70, rotate=90] {mean--min centering};
		
		% ---------- ROW 2: hatDelta -> Vhat0 ----------
		% \hat{\Delta}
		% Vector m
		\begin{scope}[shift={($(rowtwo)+(0.5cm,0.3cm)$)}]
			\draw[rounded corners=1.2pt, line width=0.6pt, black!70] (0,0) rectangle (\cell+0.4cm,-3*\cell);
			\node[] at (0.7*\cell+2,-0.5*\cell) {$\widehat V_\theta(\textcolor{C1}{s_0})$};
			\node[] at (0.7*\cell+2,-1.5*\cell) {$\widehat V_\theta(\textcolor{C2}{s_0})$};
			\node[] at (0.7*\cell+2,-2.5*\cell) {$\widehat V_\theta(\textcolor{C3}{s_0})$};
			\path (0.5*\cell,-1.5*\cell) coordinate (mcenter);
			
			\node[black!80, align=center] at (1.5*\cell+ 1.8cm, -3*\cell-0.6cm)
			{Refinement: \\ $\widehat V_\theta(\textcolor{C1}{s_{0}}) := O(\textcolor{C1}{s_{0}}) - \min\big(O(\textcolor{C1}{s_{0}}),O(\textcolor{C2}{s_{0}}),O(\textcolor{C3}{s_{0}})\big)$};
		\end{scope}
		
%		\begin{scope}[shift={($(rowtwo)+(0.5cm,0.3cm)$)}]
%			\draw[rounded corners=1.4pt, line width=0.6pt, black!70] (0,0) rectangle (3*\cell, -3*\cell);
%			\foreach \i in {1,2} {
%				\draw (0,-\i*\cell) -- (3*\cell,-\i*\cell);
%				\draw (\i*\cell,0) -- (\i*\cell,-3*\cell);
%			}
%			\node[C1] at (0.5*\cell,-0.5*\cell) {$\hat\Delta_{11}$};
%			\node[C2] at (1.5*\cell,-0.5*\cell) {$\hat\Delta_{12}$};
%			\node[C3] at (2.5*\cell,-0.5*\cell) {$\hat\Delta_{13}$};
%			\node[C1] at (0.5*\cell,-1.5*\cell) {$\hat\Delta_{21}$};
%			\node[C2] at (1.5*\cell,-1.5*\cell) {$\hat\Delta_{22}$};
%			\node[C3] at (2.5*\cell,-1.5*\cell) {$\hat\Delta_{23}$};
%			\node[C1] at (0.5*\cell,-2.5*\cell) {$\hat\Delta_{31}$};
%			\node[C2] at (1.5*\cell,-2.5*\cell) {$\hat\Delta_{32}$};
%			\node[C3] at (2.5*\cell,-2.5*\cell) {$\hat\Delta_{33}$};
%			
%			\node[font=\tiny, black!80, align=center] at (1.5*\cell+ 0.5cm, -3*\cell-0.6cm)
%			{$\hat\Delta_{ij}=\begin{cases}\Delta_{ij}-m_i,&\Delta_{ij}>0,\\[2pt] m_i-\Delta_{ij},&\text{else.}\end{cases}$};
%		\end{scope}
		
		% Arrow to column mean + step marker (3)
		\filldraw[white] ($(rowtwo)+(0.1cm,1.0*\cell-0.0cm)$) circle (6pt);
		\node[draw, circle, inner sep=1.6pt, black!70] at ($(rowtwo)+(0.1cm,1.0*\cell-0.0cm)$) {\scriptsize 3};
		\node[font=\scriptsize, black!60, anchor=west] at ($(rowtwo)+(0.3cm,1.0*\cell-0.0cm)$) {\scshape Refine Offset};

		\draw[-{Stealth[length=2.2mm]}, line width=0.7pt, SoftGray]
		($(rowtwo)+(3*\cell-0.3cm,-1.5*\cell+0.2cm)$) -- ++(1.0cm,0)
		node[midway, above, font=\tiny, black!70] {use};
		
		% Vector \hat V_0 (step 4)
		\begin{scope}[shift={($(rowtwo)+(3*\cell+0.8cm, 0.3cm)$)}]
			\draw[rounded corners=1.2pt, line width=0.6pt, black!70] (0,0) rectangle (\cell+87,-3*\cell);
			\node[font=\scriptsize] at (1.0*\cell+35,-0.5*\cell) {$\bar V_\theta(\textcolor{C1}{s_t}) := \widehat V_\theta(\textcolor{C1}{s_{0}}) + \Delta_\theta(\textcolor{C1}{s_t}, \textcolor{C1}{s_0})$};
			\node[font=\scriptsize] at (1.0*\cell+35,-1.5*\cell) {$\bar V_\theta(\textcolor{C2}{s_t}) := \widehat V_\theta(\textcolor{C2}{s_{0}}) + \Delta_\theta(\textcolor{C2}{s_t}, \textcolor{C2}{s_0})$};
			\node[font=\scriptsize] at (1.0*\cell+35,-2.5*\cell) {$\bar V_\theta(\textcolor{C3}{s_t}) := \widehat V_\theta(\textcolor{C3}{s_{0}}) + \Delta_\theta(\textcolor{C3}{s_t}, \textcolor{C3}{s_0})$};
			%\node[font=\scriptsize, black!80, align=center]
			%at (0.5*\cell,-3*\cell-0.35cm) {$\hat V_0(s_j)=\tfrac{1}{3}\sum_{i=1}^3 \hat\Delta_{ij}$};
			\path (0.5*\cell,-0.5*\cell) coordinate (vA);
			\path (0.5*\cell,-1.5*\cell) coordinate (vB);
			\path (0.5*\cell,-2.5*\cell) coordinate (vC);
		\end{scope}
		
		\filldraw[white] ($(rowtwo)+(3*\cell+2.15cm-0.32cm,1.0*\cell-0.0cm)$) circle (6pt);
		\node[draw, circle, inner sep=1.6pt, black!70] at ($(rowtwo)+(3*\cell+2.15cm-0.32cm,1.0*\cell-0.0cm)$) {\scriptsize 4};
		
		\node[font=\scriptsize, black!60, anchor=west] at ($(rowtwo)+(3*\cell+2.15cm-0.1cm,1.0*\cell-0.0cm)$) {\scshape Add Offset};
		
		% ===================== Connections to offsets =====================
		%\draw[C1, dashed, line width=0.8pt] (vA) to[out=180,in=0] (Aoff);
		%\draw[C2, dashed, line width=0.8pt] (vB) to[out=180,in=0] (Boff);
		%\draw[C3, dashed, line width=0.8pt] (vC) to[out=180,in=0] (Coff);
		
		% ===================== Panel framing (subtle) =====================
		\coordinate (rightTop) at ($(rightcol)+(0,0.55cm)$);
		\coordinate (rightBot) at ($(rowtwo)+(3*\cell+3.15cm+\cell, -3*\cell-0.55cm)$);
		
		\begin{pgfonlayer}{background}
			% Left panel
			%\node[fill=black!3, rounded corners=2pt, inner sep=6pt, draw=black!12, fit=(before) (after)] {};
			\draw[fill=black!3, draw=black!12, rounded corners=2pt]
			(-1,-4) rectangle (5.5,2.6); % lower-left and upper-right corners
			% Right panel
			%\node[fill=black!3, rounded corners=2pt, inner sep=6pt, draw=black!12, fit=(rightTop) (rightBot)] {};
			\draw[fill=black!3, draw=black!12, rounded corners=2pt]
			(6,-4) rectangle (12.7,2.6); % lower-left and upper-right corners
		\end{pgfonlayer}
		
	\end{tikzpicture}
	\caption{\textbf{Trajectory Ranking.} When training batches contain samples from more than one episode, the initialization $\tilde V(s_0)=0$ for each trajectory $\tau_i$ is not correct. The trajectories need to be ranked relative to each other by adding an offset that is calculated with $\Delta(s_i,s_j)$. Note that \textcolor{C1}{$s_0$}, \textcolor{C2}{$s_0$}, \textcolor{C3}{$s_0$} are start states of \textcolor{C1}{$\tau_1$}, \textcolor{C2}{$\tau_2$}, \textcolor{C3}{$\tau_3$} indicated by color. For simplicity, assume in this figure that start states are only present at $t=0$ for each rollout, so we can think about each $\tau$ as one episode.}
	\label{fig:relative_offsets}
\end{figure}
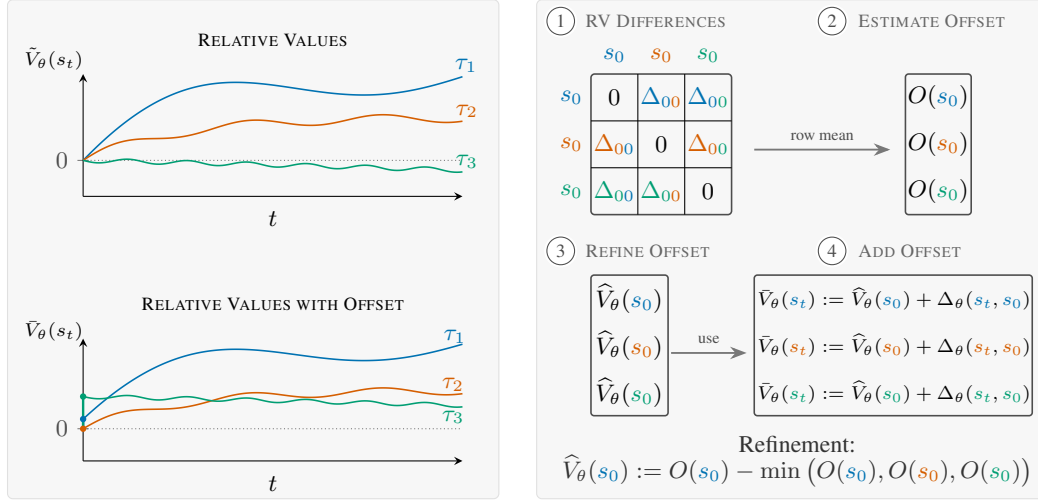

\subsection{Trajectory Ranking}
\label{sec:traj-ranking}

Given $M$ training rollouts $\{\tau^{(m)}\}_{m=1}^M$ with states $\tau^{(m)}=(s^{(m)}_0,\ldots,s^{(m)}_{T})$ and done flags 
$\{d^{(m)}_t\}_{t=0}^{T-1}$ indicating whether $s^{(m)}_{t+1}$ begins a new episode. Define
\begin{equation}
	\mathcal{K}^{(m)} := \{0\}\,\cup\,\{\,t\in\{1,\ldots,T\}: d^{(m)}_{t-1}=1\,\},
\end{equation}
as the index set where either a sub-episode or new episode starts. In the following, we refer to these selected states as \emph{start states}.
Enumerate all $N=\sum_m|\mathcal{K}^{(m)}|$ start states as $\{s^{(n)}_{\text{start}}\}_{n=1}^N$, where each $s^{(n)}_{\text{start}}=s^{(m)}_k$ for some $k\in\mathcal{K}^{(m)}$.
Let $\Delta_\theta:\mathcal{S}\times\mathcal{S}\to\mathbb{R}$ be the learned (noisy) difference model and define the $N\times N$ matrix of pairwise differences
\begin{equation}
	\Delta_{ij} := \Delta_\theta\!\left(s^{(i)}_{\text{start}},\,s^{(j)}_{\text{start}}\right),\qquad 1\le i,j\le N.
\end{equation}
To obtain an offset estimation that is robust to prediction noise and is identifiable up to a batchwise constant (which is sufficient for ranking), we estimate all start state offsets $O$ via row-wise averaging and then subtract the batch minimum to get non-negative values with ranking
\begin{equation}
	O(s^{(n)}_{\text{start}}) := \frac{1}{N}\sum_{j=1}^N \Delta_{nj},
	\quad \quad 
	\widehat V_\theta(s^{(n)}_{\text{start}}) := O(s^{(n)}_{\text{start}}) - \min_{1\le \ell \le N} O(s^{(\ell)}_{\text{start}})\;.
\end{equation}

Finally and to avoid overly complex notation, for any state $s$ in a rollout, use the most recent \emph{start state} in that same rollout, call it $s_{\text{start}}$.
Set the relative value for $s$ as
\begin{equation}
	\bar V_\theta(s) \;=\; \widehat V_\theta(s_{\text{start}})+\Delta_\theta(s,s_{\text{start}}),
	\label{eq:offset_init}
\end{equation}
and use these values for R-GAE. In that way, each state gets its episode-specific offset.

\section{Training Objective}
The PPO objective uses relative advantages $\tilde A_t$ and the usual clipped surrogate:
\begin{equation}
	\mathcal L_{\text{policy}}(\theta)=
	\mathbb E_t\big[\min(r_t(\theta)\tilde A_t,\; \mathrm{clip}(r_t(\theta),1-\epsilon,1+\epsilon)\tilde A_t)\big],
	\quad
	r_t(\theta)=\frac{\pi_\theta(a_t\mid s_t)}{\pi_{\text{old}}(a_t\mid s_t)}.
\end{equation}
The critic loss and entropy bonus are
\begin{equation}
	\label{eq:rv_critic_loss}
	\mathcal L_{\text{critic}}(\theta)=\mathbb E_{(i,j)\sim\mu}\big[\big(\Delta_\theta(s_i,s_j)-y_{ij}^{(n)}\big)^2\big],\qquad
	\mathcal L_{\text{ent}}(\theta)=-\mathbb E_t[H(\pi_\theta(\cdot\mid s_t))],
\end{equation}
where we use the n-step value target (see Eq. \ref{eq:n-step-target}) and finally optimize the combined loss
\begin{equation}
	\mathcal L(\theta) = -\mathcal L_{\text{policy}}(\theta) + c_v \mathcal L_{\text{critic}}(\theta) + c_e \mathcal L_{\text{ent}}(\theta).
\end{equation}
\subsection{Network Architecture}
For all Atari experiments we use the exact same architecture as PPO \cite{schulman2017proximal}.
The policy and relative value function share the CNN encoder $f_{\rm enc}(s)\in\mathbb{R}^d$ and then split their computation by using a single linear layer for their respective outputs.

\paragraph{Relative Critic.}
Relative values are obtained by applying the same encoder 
$f_{\rm enc}(s)$ to both states and projecting the difference of their embeddings. We do not use an additional target encoder or stop-gradients. Formally, the value difference for a state pair $(s_i,s_j)$ is given by 
\begin{equation}
	\Delta_{\theta}(s_i,s_j) = \Phi\bigl(\,f_{\rm enc}(s_i) - f_{\rm enc}(s_j)\bigr),
\end{equation}
where $\Phi$ is the projection head. For fair comparisons, our experiments use a single learned vector $w \in \mathbb R^{d}$ without bias term to ensure antisymmetry $\Delta_\theta(s_i,s_j) = - \Delta_\theta(s_j,s_i)$ and $\Delta_\theta(s_i,s_i) = 0$ by design. It is also feasible to build $\Phi$ as an non-linear MLP that is antisymmetric in nature (e.g. by using tanh activations and no bias term in linear layers). But so far, we have not observed additional improvements by using such non-linear heads.

\section{Experiments}

We evaluate \emph{Relative Value Learning} (RV) as a drop-in critic for on-policy policy-gradient methods on the Arcade Learning Environment (ALE) for Atari. Observations follow the standard PPO preprocessing: random no-op resets, frame skip with max-over-two, grayscale resizing to $84{\times}84$, stacking $m{=}4$ frames, and input scaling to $[0,1]$. Networks use orthogonal initialization with a small policy logit scale ($0.01$) and unit scale for the value head, matching widely used PPO implementations. We run with EnvPool~\cite{weng2022envpool} for high-throughput environment simulation and train for $40$M frames ($10$M environment steps). For each game we use $10$ independent seeds and report performance as the average score over the last $100$ training episodes. All hyperparameters stay identical over 49 games and are reported in Appendix \ref{sec:hyperparameters}.
\begin{table}[H]%htbp
	\tiny
	\centering
	\caption{\textbf{PPO+RV (ours) is competitive with PPO and DAE.} Mean final scores (last 100 episodes) with standard deviation of PPO, DAE and our method (PPO+RV) after 40\,M game frames.}
	\label{tab:atari_final_scores}
	\begin{tabular}{
			l
			S[table-format=7.1,
			table-space-text-post={\,$\pm$\,68362.1}]
			S[table-format=7.1,
			table-space-text-post={\,$\pm$\,683262.1}]
			S[table-format=7.1,
			table-space-text-post={\,$\pm$\,683262.1}]
			%S[table-format=7.1]
			%S[table-format=7.1]
		}
		\toprule
		Game               & {PPO } \citep{schulman2017proximal} & {DAE \citep{pan2022direct}} & {PPO + RV (ours)} \\
		\midrule
		Alien                  &  1850.3$\pm$376.8    & 1372.7$\pm$349.3 & 1748.8$\pm$408.7\\
		Amidar                 &  674.6$\pm$108.5   & 394.6$\pm$121.9 &  504.7$\pm$108.8\\
		Assault                &  4971.9$\pm$642.4   & 2205.1$\pm$362.3  & 3901.7$\pm$219.2\\
		Asterix                &  4532.5$\pm$1570.7   & 3750.1$\pm$522.9 & 3862.2$\pm$678.2\\
		Asteroids              &  2097.5$\pm$88.8   & 1392.3$\pm$62.3 & 1489.5$\pm$91.5\\
		Atlantis               &  2311815.0$\pm$420555.5 & 2888011.1$\pm$225889.1	 & 3101686.1$\pm$451117.9\\
		BankHeist              &  1280.6$\pm$2.7   & 257.8$\pm$193.2 & 1209.0$\pm$77.5\\
		BattleZone             &  17366.7$\pm$1045.0  & 16302.0$\pm$2042.5 & 21780.0$\pm$1516.1\\
		BeamRider              &  1590.0$\pm$276.3    & 1729.9$\pm$172.7  & 2044.0$\pm$396.3\\
		Bowling                &  40.1$\pm$13.8	   & 36.1$\pm$9.1 & 46.3$\pm$7.9\\
		Boxing                 &  94.6$\pm$0.9 	   & 25.9$\pm$9.5  & 94.8$\pm$0.8\\
		Breakout               &  274.8$\pm$20.2     & 234.9$\pm$28.2		 & 263.0$\pm$27.7\\
		Centipede              &  4386.4$\pm$165.6   & 3915.8$\pm$694.4 & 1226.1$\pm$104.5\\
		ChopperCommand         &  3516.3$\pm$991.6   & 1587.3$\pm$423.6 & 4435.2$\pm$965.5\\
		CrazyClimber           &  110202.0$\pm$3547.7 & 112319.5$\pm$6125.4 & 111622.2$\pm$5645.4\\
		DemonAttack            &  11378.4$\pm$2800.5  & 2477.2$\pm$344.3 & 8944.7$\pm$534.3\\
		DoubleDunk             &  -14.9$\pm$1.3    & -12.5$\pm$3.0 & -23.3$\pm$5.3\\
		Enduro                 &  758.3$\pm$31.2    & 0.0$\pm$0.0 & 1080.2$\pm$121.1\\
		FishingDerby           &  17.8$\pm$2.8     & -60.4$\pm$7.2 & 19.8$\pm$5.2\\
		Freeway                &  32.5$\pm$0.3     &  19.5$\pm$13.7 & 31.9$\pm$0.2\\
		Frostbite              &  314.2$\pm$4.9    &  367.9$\pm$278.2 & 522.6$\pm$399.0\\
		Gopher                 &  2932.9$\pm$1870.6   &  1137.2$\pm$156.8 & 3719.1$\pm$249.5\\
		Gravitar               &  737.2$\pm$150.9    &  443.5$\pm$50.3 & 1441.0$\pm$561.4\\
		IceHockey              &  -4.2$\pm$0.3     &  -5.1$\pm$0.5	 & -3.9$\pm$0.4\\
		Jamesbond              &  560.7$\pm$94.9    &  507.8$\pm$21.7 & 568.9$\pm$43.7\\
		Kangaroo               &  9928.7$\pm$7089.9   &  1331.0$\pm$1060.8 & 5649.0$\pm$2940.0\\
		Krull                  &  7942.3$\pm$555.1    &  9034.0$\pm$774.6 & 8870.9$\pm$473.1\\
		KungFuMaster           &  23310.3$\pm$2251.9   &  20535.3$\pm$2810.7 & 24483.6$\pm$6794.6\\
		MontezumaRevenge       &  42.0$\pm$57.4      &  0.1$\pm$0.3 & 1.4$\pm$3.4\\
		MsPacman               &  2096.5$\pm$157.1   &  2501.0$\pm$600.9 & 1721.2$\pm$230.1\\
		NameThisGame           &  6254.9$\pm$160.9   &  6016.3$\pm$283.3 & 6685.3$\pm$938.3\\
		Pitfall                &  -32.9$\pm$19.0    &  -12.0$\pm$23.1 & -27.3$\pm$26.3\\
		Pong                   &  20.7$\pm$0.2     &  20.7$\pm$0.3 & 16.8$\pm$3.1\\
		PrivateEye             &  69.5$\pm$55.2     &  86.0$\pm$14.9 & 93.3$\pm$17.6\\
		Qbert                  &  14293.3$\pm$293.1  &  11119.6$\pm$3465.8 & 15261.6$\pm$502.1\\
		Riverraid              &  8393.6$\pm$385.4   &  3150.8$\pm$549.9 & 9465.8$\pm$1062.5\\
		RoadRunner             &  25076.0$\pm$10851.2  &  16146.3$\pm$3074.0 & 43346.3$\pm$6741.1\\
		Robotank               &  5.5$\pm$0.7      &  6.9$\pm$2.7 & 19.5$\pm$3.1\\
		Seaquest               &  1204.5$\pm$481.1   &  2049.8$\pm$430.3 & 1659.7$\pm$209.4\\
		SpaceInvaders          &  942.5$\pm$266.1    &  914.9$\pm$219.4 & 712.3$\pm$134.5\\
		StarGunner             &  32689.0$\pm$949.8  &  5849.2$\pm$740.4 & 24830.8$\pm$10990.8\\
		Tennis                 &  -14.8$\pm$3.8    &  -17.6$\pm$0.9 & -31.7$\pm$3.4\\
		TimePilot              &  4342.0$\pm$331.1   &  7252.7$\pm$1173.6 & 10212.7$\pm$492.0\\
		Tutankham              &  254.4$\pm$42.5    &  196.0$\pm$26.2 & 224.2$\pm$23.1\\
		UpNDown                &  95445.0$\pm$21584.1  &  85438.4$\pm$19664.6 & 113991.7$\pm$21370.6\\
		Venture                &  0.0$\pm$0.0      &  0.0$\pm$0.0 & 5.5$\pm$15.6\\
		VideoPinball           &  37389.0$\pm$15876.6  &  23958.6$\pm$3936.0 & 138564.8$\pm$77425.7\\
		WizardOfWor            &  4185.3$\pm$1029.4   &  4161.3$\pm$696.0 & 5084.1$\pm$522.6 \\
		Zaxxon                 &  5008.7$\pm$1261.3   &  5612.2$\pm$1712.5 & 845.8$\pm$1567.9\\
		\bottomrule
	\end{tabular}
\end{table}
\paragraph{Compute Resources.}
Each $40$M-frame run completes in approximately $65$ minutes on a single A100 GPU with $12$ CPU cores. Across all $49$ games and $10$ seeds ($490$ runs total), this corresponds to $530$ GPU-hours or 22.10 A100-days.

\begin{figure}[H]
	\centering
	\includegraphics[width=141mm]{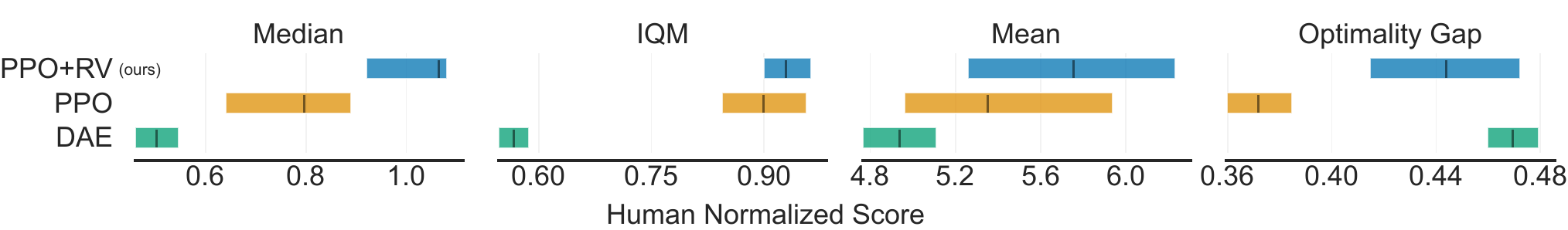}
	\caption{\textbf{Comparison between Methods.} Aggregate metrics with 95\% stratified bootstrap confidence intervals \citep{agarwal2021deep}. Higher median, interquartile mean (IQM), and mean, but lower
		optimality gap indicate better performance.}
	\label{fig:metrics}
\end{figure}
\subsection{Arcade Learning Environment (ALE)}
Table~\ref{tab:atari_final_scores} presents the per-game scores of PPO \citep{schulman2017proximal} and Direct Advantage Estimation (DAE) \citep{pan2022direct} and extends the table with an additional column, \textsc{PPO+RV}. This column represents our algorithm where the absolute value critic is replaced by the RV critic. Across the benchmark, \textsc{PPO+RV} attains competitive performance: it exceeds PPO on 30 out of 49 games (or 61\%) and DAE on 37 out of 49 games (or 75\%). In addition we report aggregated metrics in Figure \ref{fig:metrics} for further analysis.
\subsection{Ablation Study}
Furthermore, we present the influence of relative value initialization in Figure \ref{fig:ablation_value_init}. In general, we see that relative ranking is needed, but sometimes actors are not bothered by the influence of $B_t$ which we speculate is due to the PPO clipping objective.
\begin{figure}
	\centering
	\includegraphics[width=132mm]{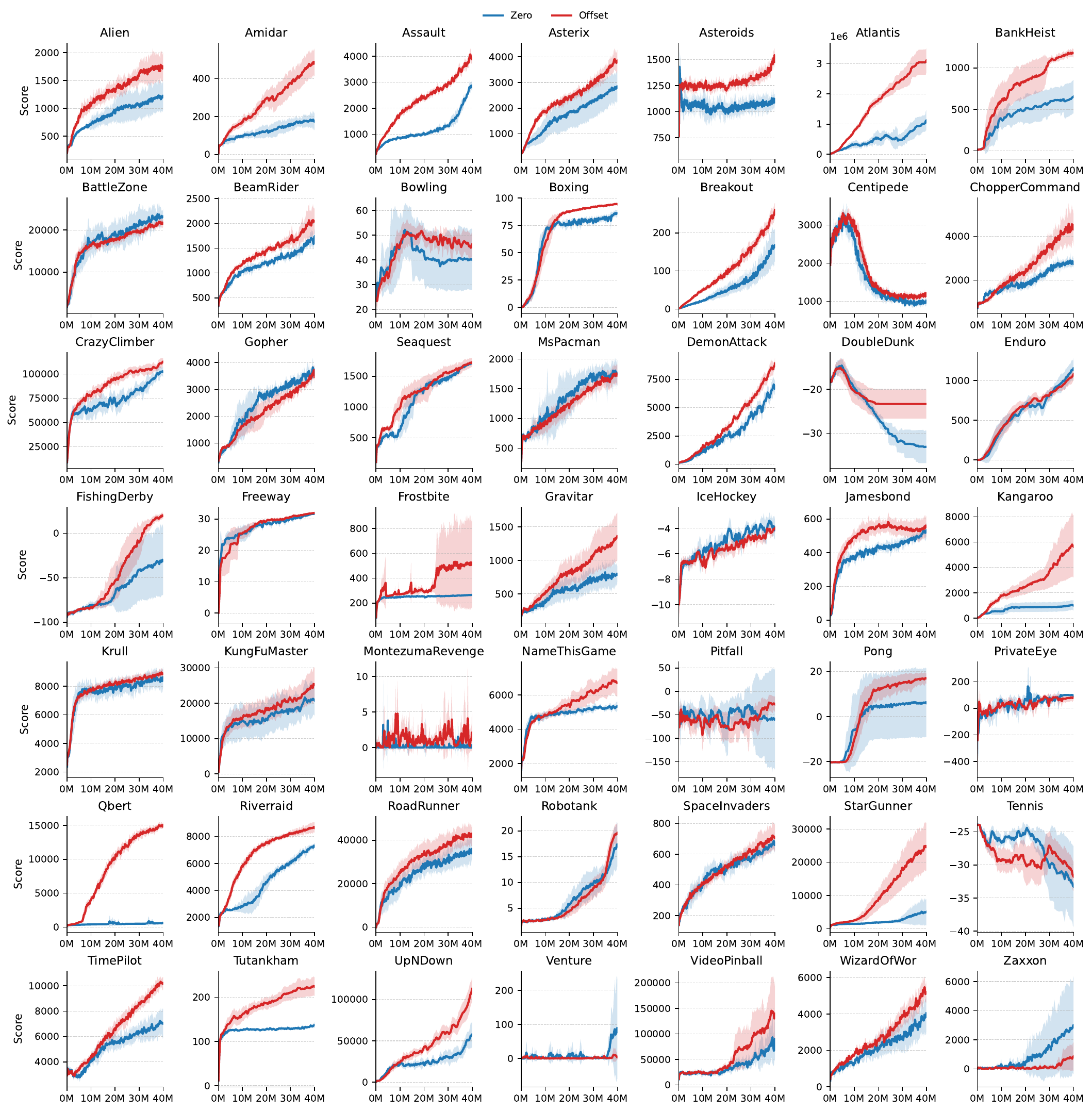}
	\caption{\textbf{Ablation for Value Initialization.} This figure compares the performance difference between zero (see Equation \ref{eq:val_init}) and offset initialization (see Equation \ref{eq:offset_init}) for relative values. Usually the proposed offset initialization is needed to improve credit assignment, but for some games the algorithm can handle zero initialization as well.}
	\label{fig:ablation_value_init}
\end{figure}
%\vspace{-3mm}
\section{Limitations}
%\vspace{-2mm}
Gauge fixing induces a trajectory-constant baseline $B_t$ in R-GAE that inflates variance for long horizons or when $\gamma\lambda\!\to\!1$, but it is only partially reduced by relative value initialization. The enforced antisymmetry improves stability but restricts critic expressivity (our near-linear difference head may underfit complex value differences). Pairwise training is $O(B^2)$, where $B$ is the batch size, in the naive form and relies on subsampling that trades compute for estimator variance. Trajectory ranking used for initialization assumes post-baseline values are on a comparable scale across trajectories. Ranking signal may become uninformative. Empirically, experiments are limited to PPO on discrete-action Atari. Generality to continuous control, off-policy regimes, and preference-based settings remains to be demonstrated.

\section{Conclusion}
\vspace{-2mm}
In this paper we established that reinforcement learning can operate on value differences rather than absolute values. We proposed \emph{Relative Value Learning} (RV), a gauge-invariant alternative to absolute critics that learns antisymmetric value differences $\Delta^\pi(s_i,s_j)=V^\pi(s_i)-V^\pi(s_j)$ as a primitive representation. On the theoretical side, we defined a pairwise Bellman operator that is a $\gamma$-contraction on bounded antisymmetric functions with a unique fixed point equal to the true value differences, and we derived well-posed bootstrapping targets (1-step/$n$-step/$\lambda$) that operate entirely on observable reward differences and non terminal pairwise terms. We further introduced trajectory ranking and reconstructed generalized advantage estimation from pairwise differences (R-GAE), showing that it ensures an unbiased policy-gradient estimator and clarifying its relationship to standard GAE. Empirically, replacing the PPO critic with our relative critic gives competitive performance across Atari while simplifying the value learning objective to the quantities that matter for control: differences rather than absolutes. We hope RV serves as a building block for algorithms that reason natively in the relative domain where decisions are actually made.

\section{Acknowledgments}
This research has been funded and supported by the Federal Ministry of Research, Technology and Space of Germany and the state of North Rhine-Westphalia as part of the Lamarr Institute for Machine Learning and Artificial Intelligence.

\bibliography{references}
\bibliographystyle{iclr2026_conference}

\appendix

\newpage
\section{Target Derivation for 1-Step Bootstrapping }
\label{sec:rv_target_derivation}

When approximating the pairwise Bellman operator from Equation~\ref{eq:pairwise-bellman-operator} by using two randomly sampled transitions
\[
\tau_i = (s_i, a_i, r_i, d_i, s_{i+1}), \quad \tau_j = (s_j, a_j, r_j, d_j, s_{j+1}),
\]
we can encounter terminal successor states that result in ill-defined targets for $\Delta(s_{i+1}, s_{j+1})$. We present four exhaustive cases that arise from the possible combinations of terminal and non-terminal successor states where $d_i,d_j \in \{0,1\}$ are terminal state indicators for $s_{i+1}$ and $s_{j+1}$. In each case, we derive an equivalent expression for $\Delta(s_{i+1}, s_{j+1})$ by only using (1) observable rewards $r_i$, $r_j$ and (2) non-terminal, well-defined value differences $\Delta_\theta(\cdot,\cdot)$. The rearrangement is then used for the bootstrapping target in Equation (\ref{eq:rv-1step-cases}) to work for sampling.

\paragraph{Case 1:} Both successor states are non-terminal ($d_i = 0, d_j = 0$). Then, the relative value 
\begin{equation}
	\Delta(s_{i+1}, s_{j+1}) = V(s_{i+1}) - V(s_{j+1}),
\end{equation}
is well-posed and directly expressible with $\Delta_\theta(s_{i+1}, s_{j+1})$.

\paragraph{Case 2:} Assume the successor $s_{i+1}$ is terminal, and $s_{j+1}$ is non-terminal ($d_i = 1, d_j = 0$). To address this, we derive a modified target:
\begin{equation}
	\begin{aligned}
		\Delta(s_{i+1},s_{j+1}) &= V(s_{i+1}) - V(s_{j+1})\\
		&=
		V(s_{i+1}) - V(s_{j+1}) + V(s_i) - V(s_i)\\
		&=
		\Delta(s_i,s_{j+1}) - V(s_i) + \underbrace{V(s_{i+1}) }_{ = 0} \\
		&=
		\Delta(s_i,s_{j+1}) - \mathbb{E}_{s_i}\left[ R_t + \gamma R_{t+1} + \dots \right]\\
		&\overset{\text{using} \: \tau_i}{\approx}
		\Delta(s_i,s_{j+1}) - r_i
	\end{aligned}
\end{equation}

\paragraph{Case 3:} Assume the successor $s_{i+1}$ is non-terminal, and is $s_{j+1}$ terminal ($d_i = 0, d_j = 1$). To address this, we derive a modified target:
\begin{equation}
	\begin{aligned}
		\Delta(s_{i+1},s_{j+1}) &= V(s_{i+1}) - V(s_{j+1})\\
		&=
		V(s_{i+1}) - V(s_{j+1}) + V(s_j) - V(s_j)\\
		&=
		\Delta(s_{i+1},s_j) + V(s_j) - V(s_{j+1}) \\
		&=
		\Delta(s_{i+1},s_j) + \mathbb{E}_{s_j}\left[ R_t + \gamma R_{t+1} + \dots \right]\\
		&\overset{\text{using} \: \tau_j}{\approx}
		\Delta(s_{i+1},s_j) + r_j
	\end{aligned}
\end{equation}

\paragraph{Case 4:} Assume that both successors are terminal ($d_i = 1, d_j = 1$). We derive a modified target:
\begin{equation}
	\begin{aligned}
		\Delta(s_{i+1},s_{j+1}) &= V(s_{i+1}) - V(s_{j+1})\\
		&=
		V(s_{i+1}) - V(s_{j+1}) + V(s_i) - V(s_i) + V(s_j) - V(s_j)\\
		&=
		\Delta(s_i,s_j) - \mathbb{E}_{s_i}\left[ R_t + \gamma R_{t+1} + \dots \right] + \mathbb{E}_{s_j}\left[ R_t + \gamma R_{t+1} + \dots \right] \\
		&\overset{\text{using} \: \tau_i,\tau_j}{\approx}
		\Delta(s_i,s_j) - r_i + r_j \\
	\end{aligned}
\end{equation}

For better training stability, we simply say that
\begin{equation}
	\begin{aligned}
		\Delta(s_{i+1},s_{j+1}) &= V(s_{i+1}) - V(s_{j+1})\\
		&=
		0\\	
	\end{aligned}
\end{equation}
since the value of both terminal states, by using MDP definition, is equal to zero.

\newpage
\section{Gauge freedom}
\begin{lemma}[Constant-offset gauge invariance]\label{lem:constant-gauge}
	Suppose $V^\pi$ satisfies the Bellman equation $V^\pi~=~r_\pi + \gamma P^\pi V^\pi$.
	For any $c\in\mathbb{R}$, define the shaped reward and shifted value
	\begin{equation}
		r'_\pi(s) \;:=\; r_\pi(s) + (1-\gamma)c,
		\qquad
		V^{\pi\prime}(s) \;:=\; V^\pi(s) + c.
	\end{equation}
	Then $V^{\pi\prime}$ satisfies the transformed Bellman equation
	\begin{equation}
		V^{\pi\prime} \;=\; r'_\pi + \gamma P^\pi V^{\pi\prime}.
	\end{equation}
	Moreover, if one equivalently defines the reward shaping $r'(s,a)=r(s,a)+(1-\gamma)c$ so that
	$r'_\pi(s)=\mathbb{E}_{a\sim\pi(\cdot\mid s)}[r'(s,a)]$, then action preferences and pairwise
	differences are invariant:
	\begin{equation}
		A^{\pi\prime}(s,a)=A^\pi(s,a)
		\quad\text{and}\quad
		\Delta^{\pi\prime}(s_i,s_j)=\Delta^\pi(s_i,s_j)\ \text{for all }(s_i,s_j)\in S\times S.
	\end{equation}
\end{lemma}
\begin{proof}
	By direct calculation,
	\begin{align*}
		r'_\pi + \gamma P^\pi V^{\pi\prime}
		&= \bigl(r_\pi + (1-\gamma)c\bigr) + \gamma P^\pi(V^\pi + c) \\
		&= \bigl(r_\pi + \gamma P^\pi V^\pi\bigr) + \bigl((1-\gamma)c + \gamma c\bigr) \\
		&= V^\pi + c \\
		&= V^{\pi\prime},
	\end{align*}
	so $V^{\pi\prime}$ satisfies the transformed Bellman equation. For invariance:
	with $r'(s,a)=r(s,a)+(1-\gamma)c$,
	\begin{equation}
		Q^{\pi\prime}(s,a) = r'(s,a) + \gamma\,\mathbb{E}_{s'}[V^{\pi\prime}(s')]
		= \bigl(r(s,a)+\gamma\,\mathbb{E}_{s'}[V^\pi(s')]\bigr) + c
		= Q^\pi(s,a) + c.
	\end{equation}
	Thus $A^{\pi\prime}(s,a)=Q^{\pi\prime}(s,a)-V^{\pi\prime}(s)=A^\pi(s,a)$, and
	$\Delta^{\pi\prime}(s_i,s_j)=(V^\pi(s_i)+c)-(V^\pi(s_j)+c)=\Delta^\pi(s_i,s_j)$.
\end{proof}

Constant offsets are a special case of potential-based shaping \citep{ng1999policy}.

\newpage
\section{Variance Analysis of the Relative Policy Gradient}
\label{sec:emperical_GAE_baseline}

In Section \ref{sec:mean-min-baseline}, we argue that the unknown trajectory constant $|C|$ can increase
the variance of the policy gradient estimator. Here, we provide the formal derivation supporting this claim and reference Figure \ref{fig:ablation_value_init} for empirical validation. First, recall from Lemma \ref{lem:gae-relation} that the relative advantage is given by $\tilde{A}_t = A_t + B_t$, where $B_t$ is a trajectory constant determined by the initialization offset $C$. While Corollary \ref{cor:unbiased-relative} proves that the estimator remains unbiased (i.e., the first moment is unchanged), the second moment differs.

\begin{lemma}[Variance Inflation]
	Let $g_{std} = \nabla_\phi \log \pi_\phi(a_t|s_t) A_t$ be the standard gradient estimator and $g_{rel} = \nabla_\phi \log \pi_\phi(a_t|s_t) \tilde{A}_t$ be the relative gradient estimator. The variance of the relative estimator is given by:
	\begin{equation}
		\text{Var}(g_{rel}) = \text{Var}(g_{std}) + \mathbb{E}\left[ \|\nabla_\phi \log \pi_\phi(a_t|s_t)\|^2 B_t^2 \right] + 2\mathbb{E}\left[ A_t B_t \|\nabla_\phi \log \pi_\phi(a_t|s_t)\|^2 \right]
	\end{equation}
	Crucially, the strictly positive term $\mathbb{E}[ \|\nabla_\phi \log \pi_\phi\|^2 B_t^2 ]$ scales quadratically with the trajectory offset $C^2$.
\end{lemma}

\begin{proof}
	The variance of any estimator $g$ is defined as $\text{Var}(g) = \mathbb{E}[\|g\|^2] - \|\mathbb{E}[g]\|^2$.
	From Corollary \ref{cor:unbiased-relative}, we know that $\mathbb{E}[g_{rel}] = \mathbb{E}[g_{std}] = \nabla_\phi J(\phi)$. Since the expected values are identical, the difference in variance is determined entirely by the second moment $\mathbb{E}[\|g\|^2]$.
	
	Let $u_t = \nabla_\phi \log \pi_\phi(a_t|s_t)$ denote the score function. Expanding it for the second moment gives:
	\begin{align}
		\|g_{rel}\|^2 &= \| u_t \tilde{A}_t \|^2 \\
		&= \| u_t (A_t + B_t) \|^2 \\
		%&= \| u_t A_t + u_t B_t \|^2 \\
		%&= (u_t A_t + u_t B_t)^\top (u_t A_t + u_t B_t) \\
		&= \|u_t A_t\|^2 + \|u_t B_t\|^2 + 2 (u_t A_t)^\top (u_t B_t)
	\end{align}
	Note that $A_t$ and $B_t$ are scalars, so we can factor them out:
	\begin{equation}
		\|g_{rel}\|^2 = \underbrace{\|u_t\|^2 A_t^2}_{\|g_{std}\|^2} + \|u_t\|^2 B_t^2 + 2 \|u_t\|^2 A_t B_t
	\end{equation}
	Recall from Lemma \ref{lem:gae-relation} that $B_t = (1-\gamma)C \sum (\gamma\lambda)^l$. Thus, $B_t$ is directly proportional to the initialization offset $C$.
	\begin{enumerate}
		\item \textbf{Noise Term:} The term $\mathbb{E}[\|u_t\|^2 B_t^2]$ is strictly non-negative. Since $B_t \propto C$, this noise term scales with $C^2$.
		\item \textbf{Correlation Term:} The term $2\mathbb{E}[\|u_t\|^2 A_t B_t]$ represents the correlation between the true advantage and the offset. While this term can be negative, it scales linearly with $C$.
	\end{enumerate}
	Consequently, for a sufficiently large uncorrected offset $|C|$, the quadratic noise term ($C^2$) will dominate the linear correlation term, resulting in an increase in estimator variance. This necessitates the \textit{Trajectory Ranking} strategy (from Section \ref{sec:mean-min-baseline}) to minimize $|C|$ and drive $B_t \approx 0$. Empirically, our ablation in Figure \ref{fig:ablation_value_init} validates the importance of decreasing estimator variance. 
\end{proof}

In addition to validate $B_t$ empirically, Table \ref{table:appendix_gae_baseline} presents a small example for a seven step trajectory. Given are rewards and true values to compute the true GAE for comparison to relative GAE:
\begin{align*}
	[r_0,\dots,r_5] &:= [1,1,1,1,1,1] \\ 
	[V(s_0),\dots,V(s_6)] &:= [2,3,4,5,6,7]
\end{align*}

In this example $V(s_0)=2$, so one can compute all predicted differences $B_t$ using $C=2$. Over all time steps, the observed advantage differences match exactly the prediction $B_t$, confirming the algebraic derivation from Lemma \ref{lem:gae-relation}. In this example, we rounded to two decimal digits for clean visualization.

\begin{table}[ht]
	\centering
	\caption{\textbf{Empirical Evidence for the Trajectory-constant.} This tables shows a small environment trajectory and validates the trajectory-constant baseline empirically. To compute true and relative GAE, we set $\gamma=0.9, \lambda=0.8$ and use $C=2$. }
	\begin{tabular}{lccccccc}
		\toprule
		& $t=0$ & $t=1$ & $t=2$ & $t=3$ & $t=4$ & $t=5$ & $t=6$ \\
		\midrule
		\multicolumn{8}{l}{\textbf{Values}} \\
		\midrule
		$r_t$                       & 1      & 1      & 1      & 1      & 1      & 1      & --     \\
		$V(s_t)$                			& 2      & 3      & 4      & 5      & 6      & 7      & 8      \\
		\midrule
		$\Delta(s_t,s_0)$      & --     & 1      & 2      & 3      & 4      & 5      & 6      \\
		$\tilde V(s_t)$ & 0      & 1      & 2      & 3      & 4      & 5      & 6      \\
		\midrule
		\multicolumn{8}{l}{\textbf{True \& Relative GAE} ($\gamma=0.9, \lambda=0.8$)}\\
		\midrule
		$\delta_t$              & 1.70   & 1.60   & 1.50   & 1.40   & 1.30   & 1.20   & --     \\
		$\tilde\delta_t$             & 1.90   & 1.80   & 1.70   & 1.60   & 1.50   & 1.40   & --     \\
		\midrule
		$A_t$                   & 4.73 & 4.21 & 3.63 & 2.96 & 2.16 & 1.20 & --     \\
		$\tilde A_t$                  & 5.35 & 4.79 & 4.15 & 3.41 & 2.51 & 1.40 & --     \\
		\midrule
		\multicolumn{8}{l}{\textbf{GAE Differences} ($C=2$)}\\
		\midrule
		$\tilde A_t - A_t$       & 0.61 & 0.58 & 0.52 & 0.45 & 0.34 & 0.20 & --     \\
		Predicted via $B_t$      & 0.61 & 0.58 & 0.52 & 0.45 & 0.34 & 0.20 & --  \\
		\bottomrule
	\end{tabular}
	\label{table:appendix_gae_baseline}
\end{table}

Figure \ref{fig:sample_B} generalizes the intuition of Table \ref{table:appendix_gae_baseline} by visualizing the prediction and decay of $B_t$ over a $T=128$ timestep rollout, which has the same length as in our experiments. In the figure, we also consider the same hyperparameter setting ($\gamma=0.99, \lambda=0.95$) which influences the magnitude of $B_t$. At the start $t=0$ of a rollout, one can see the absolute value of $B_0$ is large but almost constant, and inflates the GAE value. Then after some time, $B_t$ falls off and finally ($t\rightarrow$128) almost vanishes, so R-GAE $\approx$ GAE for the last few timesteps. This can be demonstrated for arbitrary $C$.
\begin{figure}
	\centering
	\includegraphics[width=0.7\textwidth]{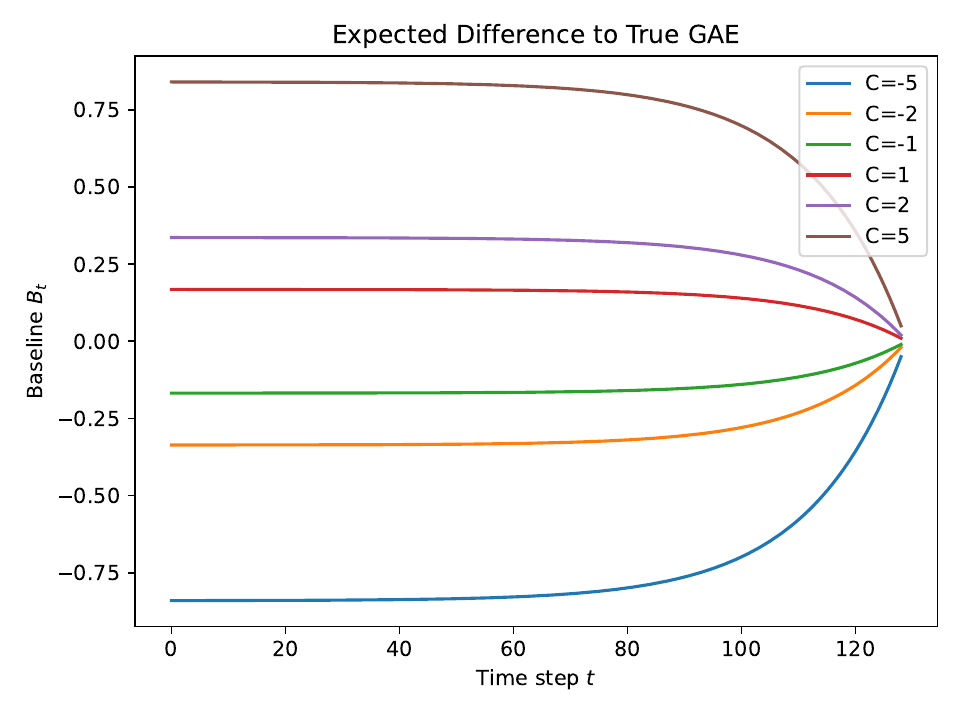}
	\caption{\textbf{Variance Visualization.} This figure shows the expected difference $B_t$ to the true GAE $A_t$ using different baseline values $C$. Note, that the difference is higher for larger $|C|$, because $\gamma$ applies a stronger absolute discount.}
	\label{fig:sample_B}
\end{figure}

\clearpage

\section{Hyperparameters}
\label{sec:hyperparameters}
We present the hyperparameter configuration for the Atari benchmark (ALE) which is used for all 49 games.

\begin{table}[h]
	\centering
	\caption{Hyperparameters for PPO+RV used in our experiments.}
	\label{tab:ppo_hparams}
	\begin{tabular}{l c}
		\toprule
		\textbf{Hyperparameter} & \textbf{Value} \\
		\midrule
		Discount factor $\gamma$            & $0.99$ \\
		GAE parameter $\lambda$             & $0.95$ \\
		N-step return target				& $5$ \\
		Clip parameter $\epsilon$           & $0.1$ \\
		Number of epochs per update         & $5$ \\
		Minibatch size                      & $128$ \\
		Number of parallel environments     & $8$ \\
		Rollout length $T$				    & $128$ \\
		Learning rate                       & $2.5 \times 10^{-4}$ \\
		Optimizer                           & Adam \\
		Adam epsilon                        & $1\times 10^{-5}$ \\
		Entropy coefficient $c_e$           & $0.01$ \\
		RV loss coefficient $c_v$           & $1.25$ \\
		RV value clipping        			& $0.15$ \\
		Gradient clipping (max-norm)        & $0.5$ \\
		\bottomrule
	\end{tabular}
\end{table}

\paragraph{Additional Hyperparameters Details.}
For training the relative critic we do not use purely random state pairing. With probability 33\%, the second state is chosen from the same episode as the reference state $s_i$, encouraging temporally coherent comparisons. Otherwise, the partner is sampled randomly in the batch. See Appendix \ref{sec:appendix_pair_ablation} for an ablation study.

\newpage
\section{Ablation on Pair Sampling}
\label{sec:appendix_pair_ablation}
In this section, we present Table \ref{tab:atari_pair_sampling_ablation} as an ablation study for different pair sampling strategies $\mu$ that can be applied for Equation \ref{eq:rv_critic_loss}. In our case, ``Biased'' means that with probability $p \in [0,1]$ the second state $s_j$ is randomly sampled from the same episode. In general, we can see that the sampling strategy has a rather small influence except when $p$ gets too large. Then there is a small decrease in performance. 

\begin{table}[H]%htbp
	\tiny
	\centering
	\caption{\textbf{Ablation of Pair Sampling.} Mean final scores (last 100 episodes) with standard deviation of our method PPO+RV with different pair sampling strategies after 40\,M game frames. The results average over five seeds.}
	\label{tab:atari_pair_sampling_ablation}
	\begin{tabular}{
			l
			S[table-format=7.1,
			table-space-text-post={\,$\pm$\,68362.1}]
			S[table-format=7.1,
			table-space-text-post={\,$\pm$\,683262.1}]
			S[table-format=7.1,
			table-space-text-post={\,$\pm$\,683262.1}]
			%S[table-format=7.1]
			%S[table-format=7.1]
		}
		\toprule
		Game               & {Random} & {Biased ($p=0.33$)} & {Biased ($p=0.66$)} \\
		\midrule
		Alien                  &  2123.4$\pm$346.0    & 1746.5$\pm$491.7 & 1802.0$\pm$600.5 \\
		Assault                &  4214.3$\pm$338.3   & 4018.8$\pm$174.1  & 3615.2$\pm$290.2\\
		Asterix                &  4303.2$\pm$1024.8   & 4495.2$\pm$1111.0 & 3821.2$\pm$461.1\\
		Atlantis               &  2683819.6$\pm$515071.2 & 2991307.2$\pm$542845.7	 & 3711740.4$\pm$431880.2\\
		BattleZone             &  22512.0$\pm$1038.6  & 22576.0$\pm$921.1 & 24616.0$\pm$1228.4\\
		BeamRider              &  2199.2$\pm$573.9    & 1966.5$\pm$130.1  & 1814.8$\pm$247.2\\
		Breakout               &  237.0$\pm$18.5     & 258.5$\pm$31.6 & 241.7$\pm$7.6\\
		DemonAttack            &  9364.3$\pm$1069.8  & 8669.9$\pm$734.8 & 7181.1$\pm$858.8\\
		Gopher                 &  3224.7$\pm$1097.3   &  3585.7$\pm$183.1 & 3626.2$\pm$338.1\\
		Gravitar               &  1414.2$\pm$759.3    &  1379.0$\pm$488.9 & 1232.0$\pm$406.9\\
		IceHockey              &  -3.7$\pm$0.6     &  -3.7$\pm$0.6	 & -3.9$\pm$0.4\\
		Krull                  &  9048.9$\pm$729.2    &  8527.1$\pm$393.7 & 8540.5$\pm$504.5\\
		MsPacman               &  1640.0$\pm$147.4   &  1783.4$\pm$125.8 & 1680.9$\pm$447.8\\
		NameThisGame           &  7000.3$\pm$1446.8  &  6880.2$\pm$937.0& 5890.7$\pm$296.8\\
		Qbert                  &  16630.4$\pm$938.0  &  15472.7$\pm$417.3 & 15631.9$\pm$789.1\\
		Riverraid              &  9941.6$\pm$1440.7   &  9271.0$\pm$1276.2 & 8717.8$\pm$478.7\\
		Robotank               &  22.5$\pm$1.0      &  21.3$\pm$2.0 & 19.7$\pm$1.5\\
		Seaquest               &  1592.6$\pm$307.4  &  1612.4$\pm$109.6 & 1236.1$\pm$321.7\\
		TimePilot              &  9701.6$\pm$435.3   & 9825.6$\pm$819.3 & 9692.8$\pm$584.0\\
		UpNDown                &  66393.6$\pm$27590.7  &  101955.1$\pm$25061.7 & 100891.3$\pm$33878.0\\
		VideoPinball           &  144597.1$\pm$89747.6  &  110187.1$\pm$57998.8 & 82002.4$\pm$16130.5\\
		WizardOfWor            &  5501.2$\pm$787.0   &  4890.0$\pm$494.9 & 5128.0$\pm$374.3 \\
		\bottomrule
	\end{tabular}
\end{table}

\newpage

\section{The Use of Large Language Models (LLMs)}
\label{sec:llms}
 Finally, we report on the use of LLMs in this section. The image in Figure \ref{fig:sample} is generated with Gemini 2.5 Flash Image (Nano Banana). Furthermore, we used large language models (LLMs) as writing assistants in the preparation of this manuscript. LLMs were employed to draft and refine text as well as to generate preliminary versions of some proof derivations. All mathematical results, derivations, and theoretical claims were independently verified and validated by the authors.

\end{document}